\documentclass[twoside]{article}

\usepackage[accepted]{aistats2023}
\usepackage[round]{natbib}

\usepackage{amsmath, amsfonts, physics}
\usepackage{amsthm, bm}
\usepackage{graphicx}
\usepackage{caption}
\usepackage{subcaption}
\usepackage{multirow}

\def\x{x}
\def\z{\bm{z}}
\def\pa{\text{pa}}
\def\e{\varepsilon}
\def\E{\bm{\mathbb{E}}}
\def\ec{\mathcal{E}}
\def\ic{\mathcal{I}}
\def\uc{\mathcal{U}}
\def\U{\bm{U}}
\def\c{c}
\def\R{\mathbb{R}}
\def\I{\bm{I}}

\def\M{\bm{M}}
\def\W{\bm{W}}
\def\id{\mathrm{id}}

\newtheorem{prop}{Proposition}

\begin{document}

%

%

\twocolumn[

\aistatstitle{NODAGS-Flow: Nonlinear Cyclic Causal Structure Learning}

\aistatsauthor{ Muralikrishnna G. Sethuraman \And Romain Lopez \And Rahul Mohan }
\aistatsaddress{Georgia Institute of Technology \And Stanford University, Genentech \And Genentech} 
\aistatsauthor{ Faramarz Fekri \And Tommaso Biancalani \And Jan-Christian H\"utter} 
\aistatsaddress{Georgia Institute of Technology \And Genentech \And Genentech, Corresponding Author}

]

\begin{abstract}
Learning causal relationships between variables is a well-studied problem in statistics, with many important applications in science. However, modeling real-world systems remain challenging, as most existing algorithms assume that the underlying causal graph is acyclic. While this is a convenient framework for developing theoretical developments about causal reasoning and inference, the underlying modeling assumption is likely to be violated in real systems, because feedback loops are common (e.g., in biological systems). Although a few methods search for cyclic causal models, they usually rely on some form of linearity, which is also limiting, or lack a clear underlying probabilistic model. In this work, we propose a novel framework for learning nonlinear cyclic causal graphical models from interventional data, called NODAGS-Flow. We perform inference via direct likelihood optimization, employing techniques from residual normalizing flows for likelihood estimation. Through synthetic experiments and an application to single-cell high-content perturbation screening data, we show significant performance improvements with our approach compared to state-of-the-art methods with respect to structure recovery and predictive performance.

\end{abstract}

\section{INTRODUCTION}
Understanding the causal relationships between interacting variables is a fundamental problem in science \citep{sachs_causal_2005,zhang_integrated_2013,segal_learning_2005} since a causal, or mechanistic understanding is fundamental to correctly predict the effects of previously unobserved interventions on a system. Such systems can be modeled using a directed graph where each variable in the system is associated with a node and the edges represent causal relationships.

With a few notable exceptions \citep{llc,richardson1996discovery, cyclic_equil, bongers2016theoretical}, most work on causal structure learning relies on the assumption that the underlying graph connecting the variables is a directed \emph{acyclic} graph (DAG). This assumption facilitates the definition of a probability distribution over the observed variables for very general functional relationships. It also provides additional regularization to the estimation problem by narrowing down the class of graphs that are compatible with the observed probability distribution (the \emph{Markov Equivalence Class}) \citep{richardson1996discovery}. However, there is compelling evidence that feedback loops are common in many real-world systems, such as those arising in gene-regulatory networks \citep{sachs_causal_2005, freimer_systematic_2022}, violating the acyclicity assumption. These networks can however be probed with a large number of interventions through recent technological advances in biological assays building upon CRISPR/Cas9 and single-cell RNA-Sequencing \citep{dixit_perturb-seq_2016, frangieh2021multimodal}, alleviating the need for the additional regularization provided by the DAG constraint. Moreover, enforcing acyclicity necessitates searching for candidate solutions over large combinatorial search spaces, complicating algorithm design. Combined, this suggests that Cyclic Causal Graphs (CCG) should be better suited to model causal semantics in this regime.

In this work, we present a novel framework for causal discovery that does not rely on the DAG assumption, but instead allows for the presence of cycles in the underlying graph, while also modeling flexible, nonlinear relationships between the observed nodes. It is based on formulating the observation model as the steady state of a discrete dynamical system \citep{llc}. This elegantly allows for cycles in the underlying graph but comes at the cost of necessitating an evaluation of the gradient of the governing functional relationships in the likelihood evaluation. We propose to employ the framework of normalizing flows \citep{papamakarios_normalizing_2021}, in particular contractive residual flows \citep{iresnet, russianroulette}, to deal with this complication. We provide a comparison of our framework, called NODAGS-Flow, to state-of-the-art structure learning methods on various graph recovery and prediction tasks. 

After discussing related work (Section~\ref{sec:rel-work}), we cover the problem setup including relevant background and modeling assumptions in Section~\ref{sec:setup}. Then, we present the proposed NODAGS-Flow framework (Section~\ref{sec:framework}), solving nonlinear cyclic causal discovery through likelihood optimization with contractive residual flows. Finally, we validate NODAGS-Flow on various synthetic benchmarks and on real-world data with genetic interventions (Section~\ref{sec:experiments}). Across the benchmarks, NODAGS-Flow beats state-of-the-art algorithms on nonlinear problems, even in the case when the underlying graph is acyclic, highlighting the practical benefits of NODAGS-Flow.

\section{RELATED WORK}
\label{sec:rel-work}

In causal discovery, the primary goal is to recover the underlying causal graph and the associated conditional probability distributions from observational and potentially interventional data. We next discuss previous approaches on acyclic and cyclic graphs, followed by the key contributions of our approach.

\subsection{Acyclic causal discovery}

Most causal discovery algorithms to date deal with the case of acyclic graphs, and they are commonly categorized into constraint-based, scored-based, and hybrid methods. Constraint-based methods such as the PC algorithm \citep{sprites,triantafillou2015constraint,heinze2018invariant} aim to recover the underlying graph through constraints given by conditional independence relations encoded by causal graphs. Most constraint-based methods suffer from poor scalability and necessitate complicated algorithm design to handle the graphical constraints.

Score-based methods such as GES \citep{Meek1997GraphicalMS,hauser2012characterization} learn the graph structure by optimizing a score function over candidate models. A popular choice of score function is given by the likelihood function in frequentist setups or the posterior likelihood in Bayesian formulations, and their regularized variants, such as the Bayesian Information Criterion (BIC). These methods often employ greedy approaches due to the super-exponential size of the search space.

More recently, the NOTEARS methodology \citep{notears} introduced a continuous constraint for limiting the search space of the optimization problem to DAGs, closing the gap between DAG learning and continuous optimization and avoiding explicit greedy searches over combinatorial structures. Several extensions followed \citep{yu2019dag, ng2020role, zheng20learning, lee2019scaling, dcdi} which allowed for learning DAGs under various assumptions on the underlying causal system. In particular, \cite{dcdi} introduced a Gumbel-Softmax parameterization of the adjacency matrix and interventional masks which extended the method to handle imperfect interventions. While the NOTEARS framework strongly simplifies algorithm design for DAG learning, it necessitates sequentially solving multiple optimization problems, which poses difficulties in applying its nonlinear extensions to larger graphs without further regularization~\citep{dcdfg}.

Hybrid methods combine both previous approaches \citep{tsamardinos2006max,solus2017consistency,wang2017permutation}. Notably, one proposed hybrid approach \citep{khemakhem2021causal} used autoregressive normalizing flows for the underlying model, but strongly relied on the acyclicity assumption to fix an ordering and on constraint-based methods to estimate the skeleton of the underlying graph.

NODAGS-Flow is a score-based method and closest in spirit to the NOTEARS family of algorithms because it starts from a simple score function and is entirely based on continuous optimization. However, it does extend to the cyclic case and by doing so avoids the need for sequential optimization to handle the DAG constraint. In fact, in the presence of interventional data, we show in Sections~\ref{sec:dagextension} and \ref{sec:expsynthetic} how it can beat the performance of NOTEARS and similar algorithms in the case where the underlying graph is a DAG.

\subsection{Cyclic causal discovery}

Cyclic causal discovery methods allow for feedback loops in the underlying causal mechanism, which complicates defining appropriate causal semantics. Early work on this topic extended constraint-based methods to this setting \citep{richardson1996discovery}, allowing the recovery of the underlying Markov Equivalence Class. However, exactly recovering cyclic graphs is more challenging than acyclic ones. For example, in the linear case, without resorting to assumptions such as faithfulness or sparsity, cyclic graphs are impossible to identify from purely observational data but can be  consistently recovered when the interventions satisfy ``pair conditions'' for ordered pairs of nodes through the LLC algorithm \citep{llc}. A more thorough treatment of establishing causal semantics for cyclic models can be found in \cite{bongers2016theoretical}. 

Other notable works are \cite{pmlr-v124-huetter20a} and \cite{cyclic_equil}, which use likelihood maximization for cyclic causal discovery, but they both are limited to either the linear case or rely on a linear approximation of the causal mechanism around the mean of the data, respectively.

Related works that are more disconnected from the literature on causal discovery directly start with modeling causal or mechanistic relationships as arising from a dynamical system and have shown promise in modeling biological systems \citep{yuan2021cellbox, nilsson2022artificial}. However, they lack a precise likelihood model.

Compared to these approaches, NODAGS-Flow provides a clearly defined and extensible likelihood model that handles nonlinear causal relationships.

\subsection{Contributions}

NODAGS-Flow endows the graph with semantics similar to those in \cite{cyclic_equil} and \cite{llc}, modeling the data as generated from the steady state of a dynamical system with an explicit noise model. However, instead of linear functional relationships, we allow for a rich class of nonlinear structural functions. Contrary to methods like NOTEARS \citep{notears} and its nonlinear extensions which necessitate solving a series of optimization problems to deal with the acyclicity constraint, NODAGS-Flow consists of only a single optimization, thus significantly simplifying algorithm design. In particular, our model naturally extends the classical notion of a Structural Equation Model (SEM) \citep{sem1, sem2} and subsumes DAG estimation in these models as a special case.


\section{PROBLEM SETUP}
\label{sec:setup}

\subsection{Cyclic Causal Models via Structural Equations}
Let $G = (V, E)$ represent a causal graph, where $V$, $E$ denote the set of vertices and edges, respectively. Each vertex $v_i \in V$ has an associated random variable $x_i$ corresponding to its observation and $x = (x_1, \ldots, x_d)$ denotes the complete vector of observations. Following the framework proposed by \cite{sem1} and \cite{sem2}, we use a \textit{Structural Equation Model} (SEM), also known as \textit{Structural Causal Model} (SCM), to represent the system.
That is, 
\begin{equation}
    x_i = f_i(\x_{\text{pa}(i)}) + \varepsilon_i \quad i = 1, \ldots, d,
    \label{eq:sem}
\end{equation}
where pa$(i) \subseteq \{1, \ldots, d\}\setminus \{i\}$ is the \textit{parent} set of $x_i$, $f_i$ encodes the functional dependence of $x_i$ on its parents, also referred to as the \textit{causal mechanism} of $x_i$. The parent-child relationships defined by the SEM encode the edges in $G$, i.e., the edge $x_j \to x_i$ exists if and only if $j \in \pa(i)$. The variables $(\varepsilon_1, \ldots, \varepsilon_d)$ are known as the \textit{disturbance} variables. By combining equation (\ref{eq:sem}) over $i=1, \ldots, d$ and writing $f = (f_1, \ldots, f_d)$, we have the following vectorized form:
\begin{equation}
    \x = f(\x) + \e.
    \label{eq:sem_vec}
\end{equation}
Additionally, the SEM also specifies a probability density $p_E(\e)$ over the disturbance variables. We assume that the system is free of \textit{confounders}, that is, the disturbance variables are independent of each other. Finally, we define $x$ as the solution to the system \eqref{eq:sem} for a random draw of $\e$.

In a classical SEM, the underlying graph is acyclic and a solution to \eqref{eq:sem} is naturally given by forward substitution along the topological ordering of the graph. Here, we instead explicitly assume that the mapping $\x \mapsto \e = (\id - f)(x)$ is invertible, where $\id$ is the identity map, and that both $(\id - f)$ and $(\id - f)^{-1}$ are differentiable. This ensures that there is a unique $\x$ that corresponds to each disturbance vector $\e$. Under these conditions, the probability density of $\x$ is well-defined and can be obtained using the change of variable formula for density functions,
\begin{equation}
    p_X(\x) = p_E\big((\id - f)(\x)\big)\big|\det J_{(\id - f)}(\x)\big|,
    \label{eq:likelihood}
\end{equation}
where $J_{(\id-f)}$ denotes the Jacobian matrix of the function $(\id - f)$ evaluated at $\x$. 

\subsection{Modeling Interventions}

One of the key aspects of inferring causal models is the ability to predict the behavior of the system under interventions. Following \cite{sprites} and \cite{sem2}, we consider surgical interventions, i.e., all incoming causal influences to the intervened-upon variables are removed. This results in a mutilated graph $\widetilde{G}$ where the intervened-upon nodes in $G$ have no incoming edges. Following the notational convention of \cite{llc}, we consider $K$ interventional experiments and denote one such experiment by $\ec_k = (\ic_k, \uc_k)$, where $\ic_k$ is the set of intervened-upon nodes and $\uc_k$ is the set of passively observed nodes. Let $\U_k \in \{0, 1\}^{d\times d}$ be a diagonal matrix such that $(\U_k)_{ii} = 1$ if and only if $v_i \in \uc_k$. Under the interventional setting $\ec_k$, the SEM now becomes
\begin{equation}
    \x = \U_k f(\x) + \U_k\e + \c,
    \label{eq:sem_int}
\end{equation}
where $\c$ denotes the value of the intervened-upon variables, i.e., $c_i = x_i$ if $i \in \ic_k$ and 0 otherwise. Equation (\ref{eq:sem_int}) corresponds to the assumption that the intervened-upon nodes are fixed and the equations for the passively observed variables remain unchanged. Similar to before, we assume that the functions $(\id - \U_k  f)$ are invertible for all $k$. From equation (\ref{eq:sem_vec}), the density function $\x$ for experiment $\ec_k$ is
\begin{align}
    p_X(\x) = {} & p_X (\x_{\ic_k})p_E\big([(\id - \U_k f)(\x)]_{\uc_k}\big) \nonumber\\
    & \quad |\det J_{(\id - \U_k f)}(\x)|,
    \label{eq:likeli_inter}
\end{align}
where $p_E\big([(\id - \U_k f)(\x)]_{\uc_k}\big)$ denotes subsetting the likelihood to only the variables in $\uc_k$.
Here, we assume surgical interventions as introduced above and that the intervention targets are known.

Given a set of interventions, we would like to learn the underlying parent-child relations in the graph by maximizing the likelihood of the generated data. This requires the computation of $|\det J_{(\id - \U_k f)}(\x)|$ to be tractable for each sample. To that end, we employ normalizing flows to model the map $\x \mapsto \e$ (see Section~\ref{sec:resflows}). Normalizing flows provide a rich class of functions for which the Jacobian determinant is easily computable.

\section{NODAGS-FLOW: RESIDUAL FLOW FOR CAUSAL LEARNING}
\label{sec:framework}

In this section, we present the individual components of the NODAGS-Flow framework, namely contractive residual flows for calculating the log-det term, neural network architectures for modeling $f$, diagonal preconditioning to enable DAG learning, and finally the full score function that is being optimized.
For ease of notation, in the following, we collect all model parameters into a single vector $\theta$ if not explicitly noted otherwise.

\subsection{Contractive Residual Flows for Causal Learning}
\label{sec:resflows}

Residual flows are a class of invertible functions of the form 
\begin{equation}
\label{eq:norm-flow}
z^{\prime} = z + g(z).
\end{equation}
The name comes from the resemblance to the structure of residual networks \citep{resnet}. We note that solving \eqref{eq:sem_vec} for $\e$, the relationship governing our causal model is $\e = \x - f(\x)$, which is of the same form as \eqref{eq:norm-flow} with $g(z) = -f(z)$. To ensure that our model is well-defined, we need the invertibility of this transformation, along with invertibility for every possible intervention in \eqref{eq:sem_int}. The Contractivity of $f$ is one constraint that guarantees this invertibility. In the following, we outline the machinery introduced by \cite{iresnet,russianroulette} to exploit this constraint for tractable generative modeling, which we adapt for structure learning.

A function $f: \R^d \to \R^d$ is said to be \emph{contractive} if there exists a constant $L < 1$ such that for any two points $\z_1, \z_2 \in \R^d$, 
$$\|f(\z_1) - f(\z_2)\| \leq L \|\z_1 - \z_2\|.$$
It then follows from the Banach fixed point theorem \citep{banachfp} that if $f$ is contractive, then the residual transformation $\id - \U_k f$ is invertible for any masking matrix $\U_k$. 

Although Banach's fixed point theorem guarantees invertibility, we have no analytical form for the inverse. However, the inverse can be obtained via fixed-point iterations. That is, starting with an arbitrary $\x_0$, repeatedly compute $\x_{k+1} = f(\x_k) + \e$ for all $k > 0$. The Banach fixed point theorem guarantees that this procedure converges. Moreover, the rate of convergence is exponential in the number of iterations $k$ and bounded by $O(L^k)$. This fixed-point iteration also provides an explicit interpretation of the causal semantics of the system in terms of a discrete dynamical system with fixed disturbances.

To efficiently approximate contractive functions and evaluate \eqref{eq:likeli_inter}, two technical challenges remain: enforcing contractivity and evaluating the log-determinant of the Jacobian. To address the first, we employ neural networks to approximate $f$ and note that a fixed Lipschitz constant can be enforced on a neural network layer by rescaling its weights by its spectral norm as shown by \cite{iresnet} and \cite{specnormlip}. The composition of multiple such Lipschitz layers is still a Lipschitz function.

To address the second challenge, we employ the unbiased estimator of the log-det term introduced in \cite{russianroulette}. Since $f$ is contractive, by extending the power series expansion of $\log (1 + x)$ to matrices, we have 
\begin{align}
    \log |\det J_{(\id-f)}(\x)| &= \log |\det(\I - J_f(\x))| \nonumber\\
    &= -\sum_{k=1}^{\infty}\frac{1}{k}\text{Tr}\Big\{J_{f}^k(\x)\Big\},
\end{align}
where $\I$ denotes the identity matrix.
The contractivity of $f$ guarantees the convergence of the above series. The trace of $J_f^k(\x)$ can be efficiently computed using the \textit{Hutchinson trace estimator} \citep{hutchtraceestimator}:
\begin{equation}
    \text{Tr}\Big\{J_{f}^k(\x)\Big\} = \mathbb{E}_w[w^{\top}J_{f}^k(\x)w],
\end{equation}
where $w$ is a random vector with zero mean and unit covariance. \cite{iresnet} evaluate the above power series by truncating it to a finite number of terms. However, this approach has the drawback of being biased. \cite{russianroulette} improve on this by adding additional randomization to this evaluation, truncating the power series at a \emph{random} cut-off $n\sim p(N)$, where $p$ is a probability distribution over natural numbers $\mathbb{N}$, and re-weighting the terms in the power series to obtain an unbiased estimator. Hence the final estimator we use in NODAGS-Flow is now given by,
\begin{equation}
    \log |\det J_{(\id - f)}(\x)| = -\mathbb{E}_{n, w}\Bigg[\sum_{k=1}^n \frac{w^{\top}J_{f}^k(\x)w}{k\cdot P(N \geq k)}\Bigg].
    \label{eq:logdetfinal}
\end{equation}
Here, we choose $n \sim \mathrm{Poi}(N)$, a Poisson distribution with intensity $N$ that we treat as a hyperparameter.



\subsection{Parametrization and Sparsity Penalization}
\label{sec:param-sparse}

A first na\"ive implementation of the causal mechanism $f$ through a Multi-Layer Perceptron (MLP) did not produce promising results due to the presence of self-cycles (dependencies from a node $v$ to itself). To address this, and to simultaneously add sparsity penalization on the dependency structure of $f$, we add a dependency mask $\M' \in \{0,1\}^{d \times d}$ with zero-diagonal that we apply via masking entries of x. That is, we introduce an MLP $g_{\theta}$ and set
\begin{equation}
    [f_{\theta}(x)]_i = [g_\theta(\M'_{i,*} \odot x)]_i, \quad i=1,\dots,d,
    \label{eq:masking}
\end{equation}
where $\odot$ denotes the Hadamard product. Similar to \citet{dcdi} and \citet{dcdfg}, to enable efficient learning of $\M'$ during training, we model its entries as draws from a Gumbel-Softmax distribution $\M' \sim \M_\theta$ \citep{jang2016categorical} with straight-through gradient estimation. Sparsity penalization is then achieved by adding $\lambda \, \E_{\M'\sim\M_\theta}[\|\M'\|_1]$ to the loss function for a regularization parameter $\lambda > 0$, where the expectation can be calculated explicitly from the parameters of $\M_\theta$. The Gumbel-Softmax parametrization $\M_\theta$ also offers access to an estimator for the underlying graph.

We note that in the special case of a 1-layer MLP, $f_\theta(x) = \sigma(\W^\top x)$ for a weight matrix $\W$ and an activation function $\sigma$, we can achieve the above more efficiently via enforcing a zero diagonal on $\W$ and direct $\ell_1$-penalization on the entries of $\W$.

\subsection{Extension to Non-Contractive DAGs via Preconditioning}
\label{sec:dagextension}

Although contractivity is sufficient for the invertibility of $\id - f$, it is not a necessary condition. Indeed, for the case of DAGs, the causal mechanism $f$ need not be contractive for $\id-f$ to be invertible, as the fixed point iterations will always converge after $d$ steps. However, $f$ being contractive is still convenient to efficiently estimate the absolute Jacobian-determinant as explained above. Via a diagonal rescaling (or preconditioning) of the model parameters, we can significantly increase the space of models that can be represented by contractive functions $f$. In particular, this includes all models whose underlying graph is a DAG, as shown in the following proposition. 

\begin{prop}[Non-contractive to Contractive]
Let $(G, f)$ represent a causal DAG and its causal mechanism. If $f$ is a non-contractive function, then there exists $\tilde{f}$ of the form $\tilde{f} = \Lambda \circ f \circ \Lambda^{-1}$, where $\Lambda$ denotes multiplication with a diagonal matrix with positive diagonal entries such that $\tilde{f}$ is contractive.    
\label{prop:ncon-con}
\end{prop}

We refer to the appendix for proof of the above proposition. Proposition 1 allows us to rewrite the SEM purely in terms of a contractive function ($\tilde{f}$) and a diagonal matrix ($\Lambda$), when the underlying graph is a DAG. That is, 
\begin{equation}
    \x = \Lambda^{-1}\circ \tilde{f} \circ \Lambda (\x) + \e. 
\end{equation}
Hence, for a given observed set $\mathcal{U}_k$, the logarithm of the determinant of the Jacobian now becomes
\begin{align}
\log|&\det J_{\Lambda^{-1} \circ (I - \U_k\tilde{f})\circ \Lambda}| \nonumber\\
&=\log | \det \Lambda^{-1} | + \log | \det \Lambda | + \log|\det J_{(I - \U_k \tilde{f})}|\nonumber\\
&=\underbrace{\log | \det \Lambda | - \log | \det \Lambda |}_{=0} + \log|\det J_{(I - \U_k \tilde{f})}|\nonumber\\
&= \log|\det J_{(I - \U_k \tilde{f})}|, 
\end{align}
which only depends on a contractive function and hence can be estimated efficiently using the procedure detailed in Section~\ref{sec:resflows}. In training the model, we treat $\Lambda$ as a learnable parameter to be optimized via the log-likelihood function. 

\subsection{Score Function for Differentiable Causal Learning}

Given a set of interventional experiments $\{\ec_k\}_{k=1}^K$ and corresponding observations, we would like to learn the graph structure as well as the underlying functions governing the parent-child relations. To that end, similar to previous work \citep{dcdi, dcdfg}, we use the log-likelihood of the not-intervened-on nodes as a score function. That is, approximating the log-det term by \eqref{eq:logdetfinal}, we consider
\begin{multline}
    \mathcal{L}\Big(\theta, f_\theta, \{\x^{(k,i)}\}_{k=1, i=1}^{M, N_k}\Big) = \\
    \sum_{k=1}^M \sum_{i=1}^{N_k}\bigg[\log p_{E,\theta}\Big([(\id - \U_k f_{\theta})(\x^{(k, i)})]_{\uc_k}\Big) \\
    - \E_{n, w} \bigg\{\sum_{r=1}^n\frac{w^{\top}\big[J_{\U_kf_{\theta}}^r\big(\x^{(i,k)}\big)\big]w}{r\cdot P(N \geq r)}\bigg\}\Bigg],
    \label{eq:score}
\end{multline}
where $\x^{(i, k)}$ denotes the $i$-th sample in the $k$-th experiment, and $p_{E,\theta}$ is parametrized as independent Gaussian distributions with learnable means and standard deviations.
Together with $\ell_1$ penalization introduced in Section~\ref{sec:param-sparse} with parameter $\lambda > 0$ and the preconditioning in Section~\ref{sec:dagextension}, inference is performed by solving the following optimization problem with stochastic optimization methods:
\begin{equation}
    \max_{\theta, \Lambda} \mathcal{L}(\theta, \Lambda^{-1} \circ f_{\theta} \circ \Lambda) - \lambda \, \E_{\M'\sim\M_\theta}[\| \M' \|_1].
\end{equation}

\section{EXPERIMENTS}
\label{sec:experiments}

\begin{figure*}[t]
    \centering
    \includegraphics[width=0.8\linewidth]{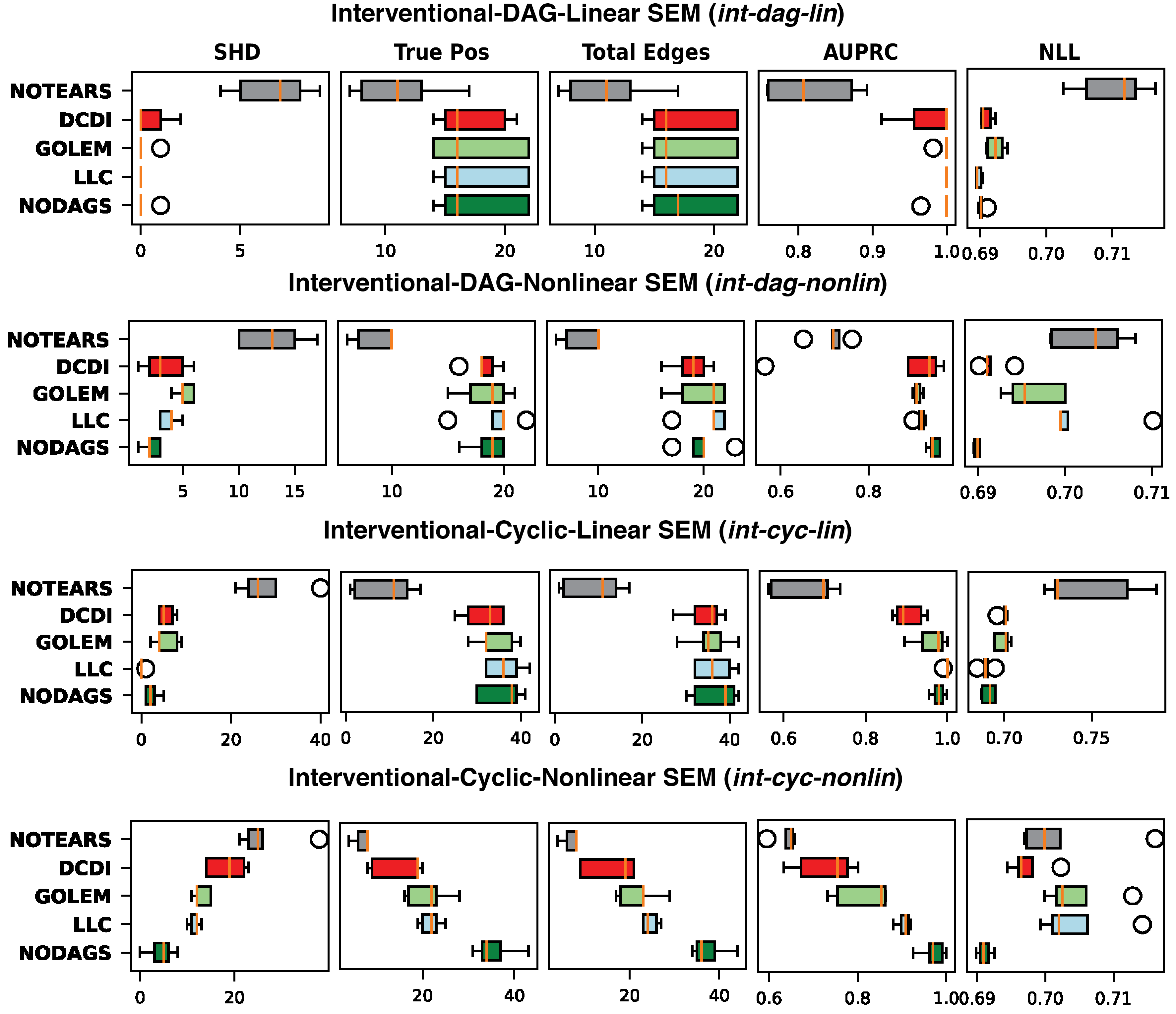}
    \caption{Performance on synthetic interventional data. The box plots show the median and inter-quartile ranges over the independent trails.  }
    \label{fig:inter-results}
\end{figure*}

We tested NODAGS-Flow on synthetic and real-world datasets. The performance of NODAGS-Flow is compared with some of the existing state-of-the-art causal discovery algorithms, LLC \citep{llc} (linear \& cyclic graphs), GOLEM \citep{ng2020role} (linear and acyclic), NOTEARS \citep{notears} (linear and acyclic), and DCDI \citep{dcdi} (nonlinear and acyclic). Of the chosen baselines, only DCDI and LLC are capable of handling interventional data out of the box, the other two algorithms were modified to allow for learning from interventions by summing over different experimental regimes and masking out loss-terms corresponding to intervened-upon nodes as in \eqref{eq:score}.

\subsection{Experiments on synthetic data}
\label{sec:expsynthetic}

We considered both observational and interventional data for the synthetic datasets. The interventions were assumed to be perfect with known targets. Each dataset was generated from graphs with $d=20$ nodes and for each intervention, 5000 observations were sampled. The observational data consists of 20,000 samples sampled from the graph. For the function $f$, we considered three different cases, namely 
(1) a linear function, $f(x) = \bm{W}^\top x$,
(2) a nonlinear function, $f = \texttt{ReLU}(\bm{W}^{\top}\x)$, a single-layer MLP with ReLU (rectified linear unit) activation, ensuring contractivity by rescaling by the operator norm,
(3) a non-contractive nonlinear function, $f = \texttt{SELU}(\bm{W}^{\top}x)$, a single-layer MLP with SELU (Scaled Exponential Linear Unit) activation, with the underlying graph being a DAG.
\begin{table}[]
    \centering
        \caption{Synthetic experiment settings.}
    \label{tab:exp-sum}
    \begin{tabular}{|c|c|c|c|}
        \hline
         \textbf{Setting} & \textbf{Interventions} & \textbf{SEM} & \textbf{Cyclic} \\
         \hline
         \emph{int-dag-lin} & True & Linear & False\\
         \textit{int-dag-nonlin} & True & Nonlinear & False\\
         \textit{int-cyc-lin} & True & Linear & True\\
         \textit{int-cyc-nonlin} & True & Nonlinear & True\\
         \textit{obs-lin} & False & Linear & False\\
         \textit{obs-nonlin} & False & Nonlinear & False\\
         \hline
    \end{tabular}
\end{table}

\begin{figure*}[t]
    \centering
    \includegraphics[width=0.8\linewidth]{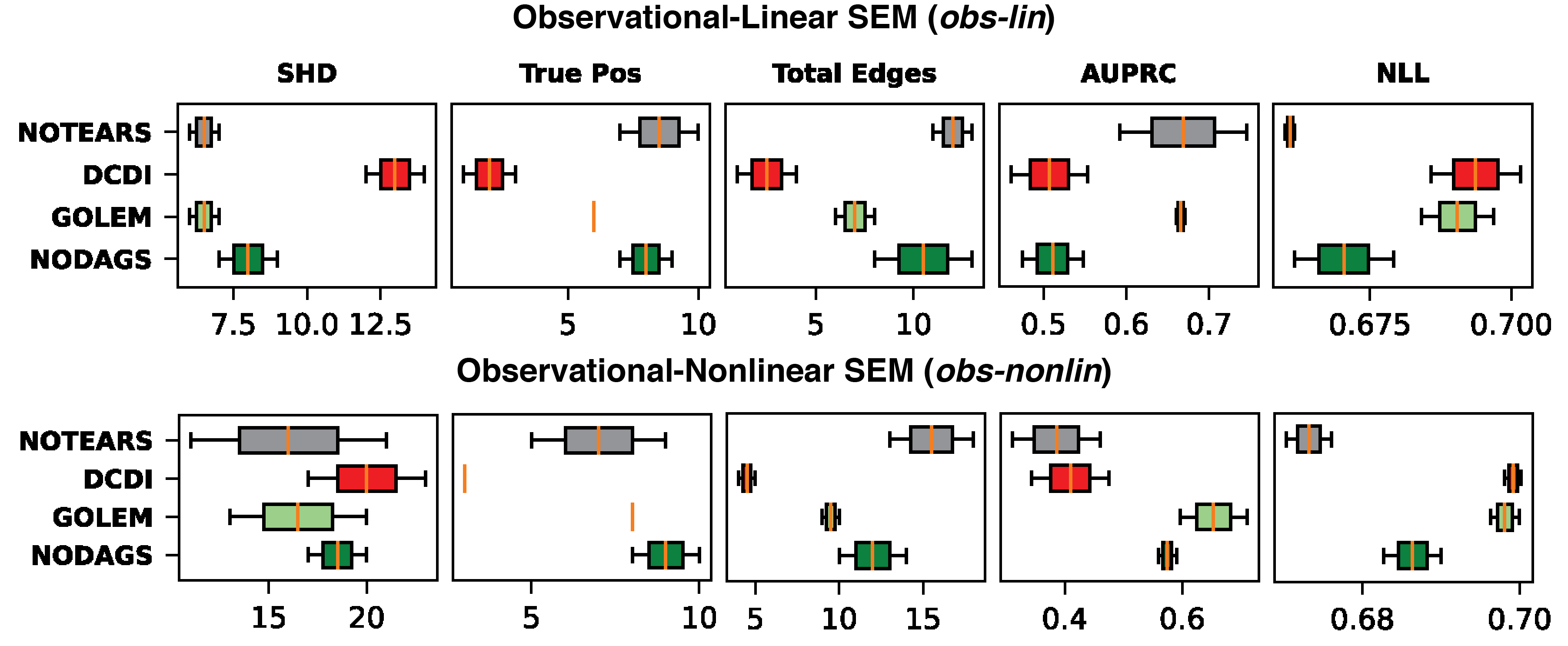}
    \caption{Performance on synthetic observational data. The box plots show the median and inter-quartile ranges over the independent trails. }
    \vspace{-1cm}
    \label{fig:obser-results}
\end{figure*}

\begin{figure*}[t]
\centering
\includegraphics[width=0.6\linewidth]{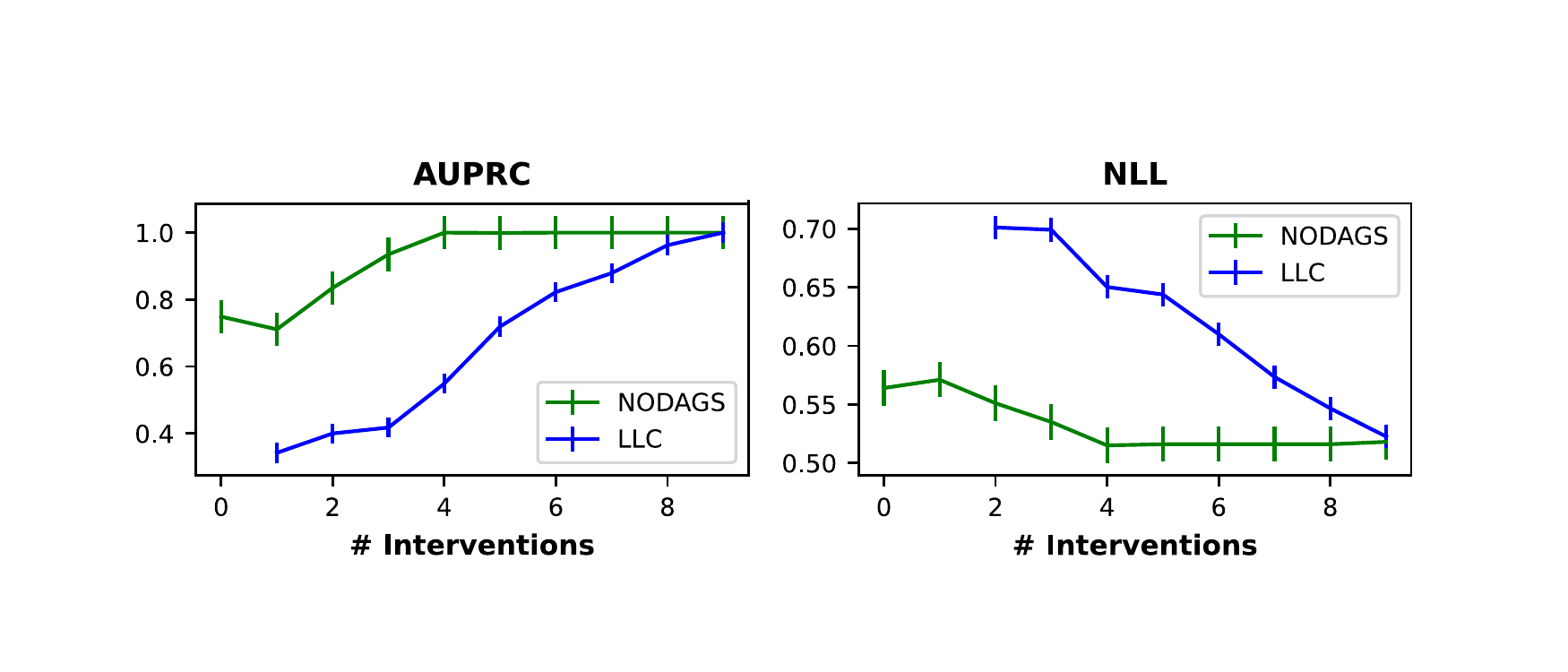}
    \vspace{-0.5cm}
\caption{Performance comparison between LLC and NODAGs-Flow as the number of interventions used for training the model is increased from 0 to 9 on a 10-node graph.}
\label{fig:perf-v-int}
\end{figure*}

In total we obtain 6 different settings for the synthetic experiments, summarized in Table \ref{tab:exp-sum}. For the setting \mbox{\textit{int-dag-nonlin}}, the causal mechanism $f$ was taken from case (3) and for the settings \mbox{\textit{int-cyc-nonlin}} and \mbox{\textit{obs-nonlin}}, $f$ was taken from case (2). The latent distribution $p_E(\e)$ was chosen as a Gaussian distribution with the same variances. The graphs were generated using an Erd\H{o}s-R\'enyi random graph model with an expected edge density of 2, allowing for cycles, whereas in case (3), we ensured acyclicity by creating a causal order and ensuring that the parents for each node always come from its predecessor in the causal order. The weight matrices were sampled from the uniform distribution, with post-scaling to ensure that the overall function is contractive for the first two settings. 

\paragraph{Performance evaluation} The performance was evaluated with respect to the following metrics: (1) \textit{Structural Hamming Distance} (SHD), (2) Total number of \textit{True Positive} edges (True Pos) predicted, (3) \textit{Total Edges} edges predicted, (4) \textit{Area Under Precision-Recall Curve} (AUPRC), and (5) holdout \textit{Interventional-NLL} (NLL) \citep{gentzel2019case}, the negative log-likelihood over unseen interventions. SHD, True Pos, Total Edges, and AUPRC measure the accuracy of the recovered graph structure whereas NLL measures the predictive power of the model over unseen interventions. For the interventional data sets, the training data consisted of single-node interventions across all nodes. For both the interventional and observational data sets, the test set consisted of interventions on two or three nodes (random with equal probability), randomly sampled from the total set of nodes. 

The results of the synthetic experiments are reported in Figures \ref{fig:inter-results} and \ref{fig:obser-results}. In both figures, the box plots show the median and the inter-quartile ranges over independent trials. In Figures \ref{fig:inter-results} and \ref{fig:obser-results}, each column shows the performance with respect to the metric stated at the top of the column for the different settings shown in Table \ref{tab:exp-sum}.  On linear interventional data (Figure \ref{fig:inter-results}, \mbox{\textit{int-dag-lin}} and \mbox{\textit{int-cyc-lin}}), NODAGS-Flow attains comparable performance to that of LLC (which was specifically designed for the interventional linear case), both in terms of graph structure recovery and the prediction of unseen interventions. In \mbox{\textit{int-dag-lin}}, GOLEM and NOTEARS match the performance of LLC and NODAGS-Flow as the setting is well specified for these models. As expected GOLEM and DCDI drop in performance when cycles are introduced (Figure \ref{fig:inter-results}, \mbox{\textit{int-cyc-lin}}).  

\begin{figure*}[t]
    \centering
    \includegraphics[width=0.94\linewidth]{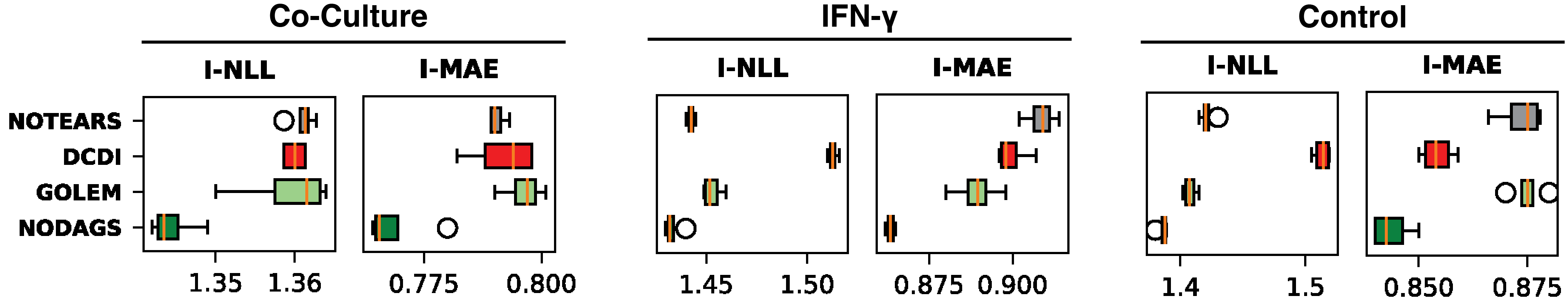}
    \caption{Performance comparison on Perturb-CITE-seq \cite{frangieh2021multimodal} data. The box plots show the median and inter-quartile ranges over the independent trails.}
    \label{fig:pert-cite-seq}
\end{figure*}

On non-linear interventional data (Figure \ref{fig:inter-results}), NODAGS-Flow performs the best when the graph contains cycles (\mbox{\textit{int-cyc-nonlin}}) followed by LLC. When the causal graph is a DAG and the causal mechanism non-contractive (Figure \ref{fig:inter-results}, \mbox{\textit{int-dag-nonlin}}), the added learnable parameter $\Lambda$ allows NODAGS-Flow to learn a non-contractive function by rescaling a contractive function $f$ by $\Lambda$. In this case, NODAGS-Flow and DCDI are the best performing models. This highlights the benefits of a larger, potentially simpler search space for structure learning provided by our approach compared to specifically enforcing DAG constraints.

On the other hand, in the observational setting, due to the inherent identifiability issues caused by not confining the search space to DAGs, NODAGS-Flow trails behind the other methods enforcing a DAG constraint, see Figure \ref{fig:obser-results}. We exclude LLC in Figure \ref{fig:obser-results} as it is incapable of handling purely observational data.

\paragraph{Scaling with Interventions}
In the previous experiments, we ensured that the models were provided with interventions over all the single nodes in the graph. Here, we test the model's capability to learn the graph structure with limited interventional information. We compare NODAGs-Flow with LLC as we increase the number of interventions provided during training in the case of a linear contractive SEM.

From Figure (\ref{fig:perf-v-int}) we can see that NODAGS-Flow requires significantly fewer interventions compared to LLC as it attains close to perfect structure recovery around 4 interventions on a $d=10$ node graph. It is also important to note that LLC cannot work on purely observational data as it subsets the data according to the performed interventions, whereas NODAGS-Flow can handle both observational and interventional data out of the box.

\subsection{Experiment on Real-World Transcriptomics Data}

Here, we present an experiment focused on learning a gene regulatory network from gene expression data with genetic interventions (Perturb-seq), a type of dataset that allows to causally investigate biology at an unprecedented scale. Recent advances \citep{dixit_perturb-seq_2016} have made it possible to perform such genetic interventions at large scales (in the order of hundreds or thousands of genes \citep{replogle2022mapping}) and be able to measure the effect of full gene expression profile on the order of hundreds of thousands of cells.

We focus on a Perturb-CITE-seq \citep{frangieh2021multimodal} dataset that investigated drivers of resistance to Immune Checkpoint Inhibitors (ICI). It contains gene expressions taken from 218,331 melanoma cells split over three different conditions, namely: (1) control (57,627 cells) , (2) co-culture (73,114 cells), and (3) interferon (IFN)-$\gamma$ (87,590 cells). Each measurement contains the identity of the target genes and the expression profiles of each gene in the genome.

Due to practical and computational limitations, we restrict our experiment to a subset of 61 genes out of approximately 20,000 genes in the genome. For interventions, we chose all the single-gene interventions corresponding to the 61 chosen genes. Each condition is considered a separate dataset and we train NODAGS-Flow and the baselines on these datasets separately. Since there is no ground-truth DAG available, we evaluate our model based on its predictive power on unseen interventions. To that end, we perform 5 splits on each of the datasets into 90\% training and 10\% test interventions. Interventional NLL (\texttt{I-NLL}) and the Interventional \textit{Mean Absolute Error} (\texttt{I-MAE}) were used as metrics to evaluate the models.\texttt{ I-MAE} was calculated as the mean of $\|f(x)-x\|_1/d$ over all observations $x$ in a hold-out dataset.

From Figure (\ref{fig:pert-cite-seq}) we can see that NODAGS-Flow outperforms all the baselines with respect to both metrics. Of all the baselines LLC seemed to attain the worst performance, we, therefore, discuss the performance comparison with LLC in more detail in the appendix. This shows that learning cycles in the graph allow for better learning of the underlying distribution and thereby improve the predictive power of the model.
Figure \ref{fig:pert-cite-seq-adj} shows the cluster map obtained from the adjacency matrix learned by NODAGS-Flow on the Co-culture partition of the Perturb-CITE-seq datasets. 

\begin{figure}
   \centering
   \includegraphics[width=\linewidth]{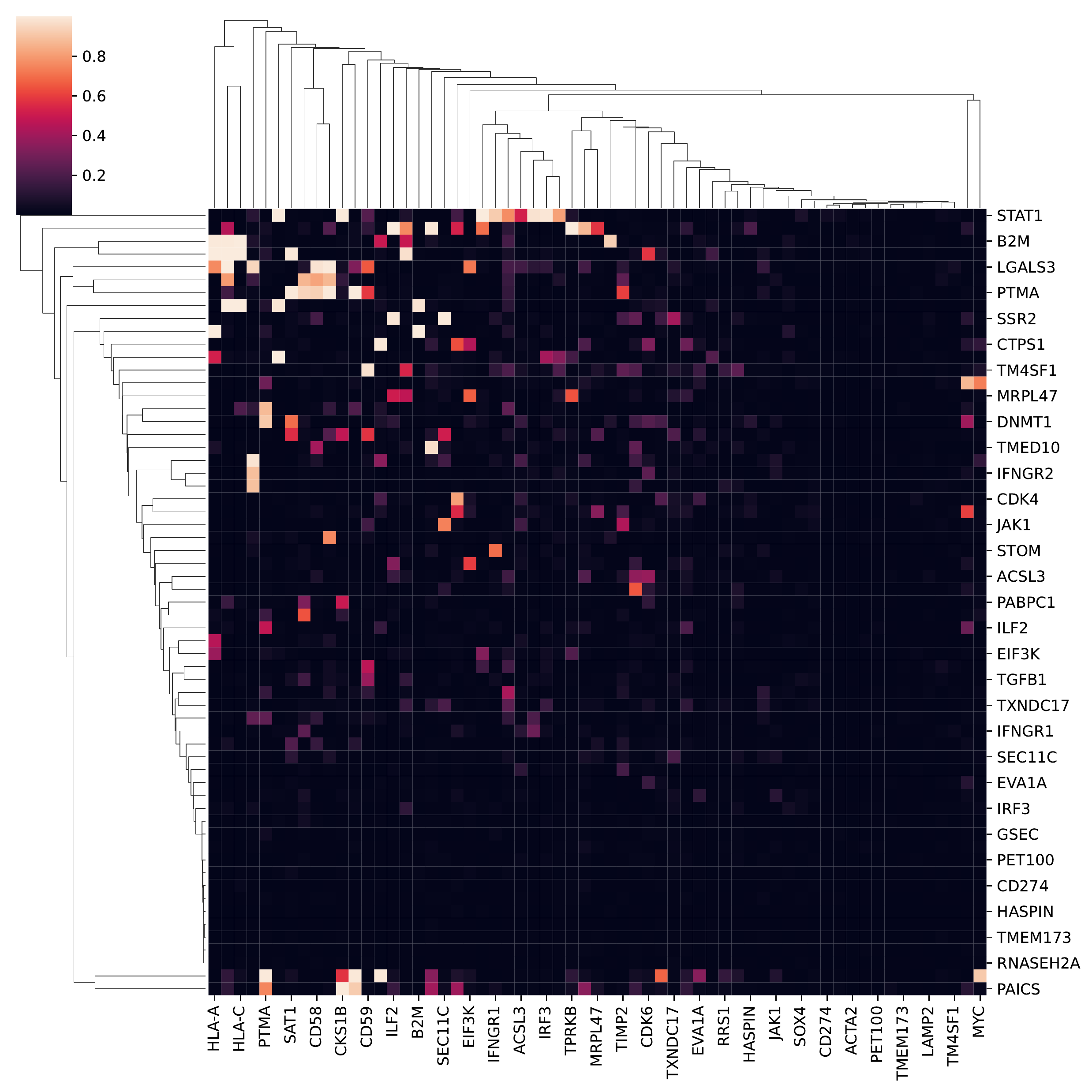}
   \caption{Adjacency matrix of the graph learned by NODAGS-Flow on the Perturb-CITE-seq data set (Co-culture).}
   \label{fig:pert-cite-seq-adj}
\end{figure}

\section{DISCUSSION}
\label{sec:discussion}

We proposed NODAGS-Flow, a novel causal discovery approach that is capable of learning nonlinear and cyclic relations between variables through a simple optimization framework, avoiding optimization problems with complex constraints such as NOTEARS. Experiments on synthetic interventional data showed matching performance with state-of-the-art methods (LLC) on linear data and superior performance when recovering nonlinear relationships, both in the case of cyclic and acyclic causal graphical models.

We also presented an application of our approach on real-world gene expression data with genetic interventions (Perturb-CITE-seq), where we learned a gene-regulatory network on 61 genes. NODAGS-Flow was able to achieve better predictive performance on unseen interventions through an interpretable, mechanistic model that allows for feedback loops. We hope that applications on more biological datasets could enhance understanding of transcriptomic regulation and aid in the design of novel perturbations.

Interesting potential extensions to our model include (1) incorporating more realistic measurement noise models, which have been shown to significantly affect the quality of transcriptomic machine-learning tools \citep{grun2014validation, lopez2018deep}, (2) explore imperfect interventions and the case where the intervention targets are unknown which is quite common in biological settings, and (3) scaling up the model to handle larger graphs, potentially by incorporating ideas from low-rank models \citep{segal_learning_2005, dcdfg}.

\section{Acknowledgments}
\label{sec:ack}

We would like to thank Kathryn Geiger-Schuller and Oana Ursu for helpful suggestions on the pre-processing and gene selection for our experiments on transcriptomics data.

\bibliography{references}

\newpage
\appendix
\onecolumn

\section*{Appendix}

The appendix is organized as follows: in Appendix \ref{app:proof} we present the proof of Proposition \ref{prop:ncon-con} followed implementation details and further details regarding the real-world experiment in appendix \ref{app:imp-det} and appendix \ref{app:perturb-cite-seq} respectively. 

\section{PROOFS}
\label{app:proof}

\subsection{Proof of Proposition 1}

In this section, we present a detailed proof of Proposition 1.

\begin{prop}[Non-contractive to Contractive]
Let $(G, f)$ represent a causal DAG and its causal mechanism. If $f: \R^d\to\R^d$, is an L-Lipschitz function, then there exists $\tilde{f}$ of the form $\tilde{f} = \Lambda \circ f \circ \Lambda^{-1}$, where $\Lambda$ denotes multiplication with a diagonal matrix with positive diagonal entries (also denoted by $\Lambda$) such that $\tilde{f}$ is contractive.    
\end{prop}

\begin{proof}
    For ease of notation, we assume that $f$ is everywhere differentiable and  denote its Jacobian by $J_f$. The general case can be proved similarly since Lipschitz functions are almost everywhere differentiable by Rademacher's theorem \citep{ambrosio2000functions}.
    
     Without loss of generality, we assume that the graph $G$ is topologically sorted along the indices $i=1,\dots,d$. If not, we can rearrange the dimensions of $f$ accordingly. With this,  the Jacobian $J_{f}$ is a strictly lower triangular matrix. For a desired contractivity constant $0 < c < 1$, we recursively define the entries of the diagonal matrix $\Lambda \in \mathbb{R}^{d \times d}$ as follows:
     \begin{equation}
         \begin{aligned}
             \Lambda_{d,d} &= 1,\\
             \Lambda_{i,i} &= \frac{d^2 L}{c} \max_{j > i} \Lambda_{j,j}, \quad \text{for } i < d.
         \end{aligned}
         \label{lambda-def}
     \end{equation}
     Defining $\tilde{f}=\Lambda \circ f \circ \Lambda^{-1}$, we have that $J_{\tilde{f}} = \Lambda J_f \Lambda^{-1}$. For any $x \in \R^d$, the $(i,j)$-th entry of $J_{\tilde{f}}$ is therefore given by
     \begin{equation}
    (J_{\tilde{f}}(x))_{i,j} = \frac{\Lambda_{i,i}}{\Lambda_{j,j}} (J_f(\Lambda^{-1}x))_{i,j}.
    \end{equation}

 Let $y = \Lambda^{-1}x$. Since $\| z \|_1 \leq \sqrt{d} \| z \|_2 $ by the Cauchy-Schwarz inequality and $\| z \|_2 \leq \| z \|_1$ for all $z \in \mathbb{R}^d$, we obtain
    \begin{equation}
        |(J_f(y))_{i,j}|
        \leq \sup_{z : \|z\|_1 \leq 1}
        \|J_{f}(y)[z]\|_1
        \leq \sup_{z : \|z\|_2 \leq 1}
        \sqrt{d}\|J_{f}(y)[z]\|_2
        \leq \sqrt{d}\norm{J_f(y)}_{\mathrm{op}} \leq \sqrt{d}L.
        \label{eq:entry-to-lipschitz}
    \end{equation}
    In turn, the entries of $J_{\tilde{f}}$ can be bounded as follows:
    for $i > j$, by combining \eqref{eq:entry-to-lipschitz} with the definition of $\Lambda$ \eqref{lambda-def}, we obtain
    \begin{align}
        |(J_{\tilde{f}}(x))_{i,j}| = 
        \frac{\Lambda_{i,i}}{\Lambda_{j,j}}|(J_f(y))_{i,j}|
        \leq \frac{1}{\Lambda_{j,j}} \sqrt{d} L \max_{k \geq i } \Lambda_{k,k}
        \leq \frac{c}{d^{3/2}}.
    \label{eq:up-bound}
    \end{align}
    For $i \leq j$, since $J_f$ and therefore $J_{\tilde{f}}$ are strictly lower triangular by definition,  $(J_{\tilde{f}}(x))_{i,j} = 0$.
    
    Finally, applying similar reasoning as in \eqref{eq:entry-to-lipschitz} and the bound in \eqref{eq:up-bound}, we obtain a bound on the operator norm of $J_{\tilde{f}}$,
    \begin{align*}
        \norm{J_{\tilde{f}}(x)}_{\mathrm{op}}
        = \sup_{z : \|z\|_2 \leq 1}
        \|J_{\tilde{f}}(x)[z]\|_2
        \leq \sup_{z : \|z\|_1 \leq \sqrt{d}}
        \|J_{\tilde{f}}(x)[z]\|_1
        = \sqrt{d}\max_j \sum_{i=1}^d |(J_{\tilde{f}}(x))_{i,j}|
        \leq d^{3/2} \frac{c}{d^{3/2}} = c < 1.
    \end{align*}
    This concludes the proof, showing that $\tilde{f}$ is contractive.
\end{proof}

\section{IMPLEMENTATION DETAILS}
\label{app:imp-det}

In this section, we present the implementation details of NODAGS-Flow, the baselines as well as the setup of the experiments. 

\subsection{NODAGS-Flow}

We consider the parametric family of neural networks (NN), denoted as $g_{\theta}$, to model the causal function $f$. As detailed in section 4.2 of the main paper, the dependency structure (parent-child relations) is encoded by introducing a dependency mask $\bm{M}^{\prime} \in \{0,1\}^{d\times d}$ in the model. $\bm{M}^{\prime}$ is then used to mask the inputs for each node as shown in equation (10) of the main paper. For the neural network architecture, we fix each hidden layer to have the same number of neurons ($=d$) and vary the number of hidden layers in the model. If the data is nonlinear we add a \texttt{ReLU} activation function to each layer of the neural network, allowing NODAGS-Flow to learn nonlinear parent-child relations. 

In order to maintain the contractivity of the neural network, the weights of each layer are rescaled by its spectral norm, similar to \citet{iresnet} and \citet{specnormlip}. This is done every time the weights are updated, that is, after every backward pass. In practice, we choose the Lipschitz constant of the neural network to be $0.9$. For computing the log Jacobian determinant, we sample the number of terms in the power series from a Poisson distribution with parameter $\sigma$ initialized to $2$. Additionally, $\sigma$ is treated as a parameter to be learned during training. During the training stage, we use the Neumann gradient series formulation of the log Jacobian determinant estimator in \cite{iresnet} as this provides a more efficient way for backpropagation over the entries of the Jacobian matrix. Whereas in the validation stage we use the standard estimator for the log Jacobian determinant. The number of hidden layers, regularization parameter $\lambda$, and the number of terms used for computing the spectral norm of the weights ($n_L$) are treated as parameters to be tuned. 

\subsection{Baseline Methods}

We now provide the implementation details of the baselines used in our experiments. The LLC algorithm proposed by \citet{llc} was reimplemented and a sparse regularization term was added in order to make LLC solve the same objective as NODAGS-Flow and the other baselines. In accordance with the method proposed by \citet{llc} we assume that the intervened nodes are independent and sampled from the standard normal distribution. Of the other baselines used only DCDI \citep{dcdi} supports interventional data out of the box. Hence NOTEARS \citep{notears} and GOLEM \citep{ng2020role} were reimplemented along the lines of \citet{dcdfg}. For NOTEARS, DCDI, and GOLEM, we threshold the adjacency matrices (probability of an edge for DCDI) with a threshold $t$ obtained by performing a binary search with $T=20$ evaluations of an acyclicity test to find the largest possible DAG from the estimated weights matrix. 

\begin{table}[t]
    \centering
    \caption{Hyperparameter spaces for all the models.}
    \vspace{0.2cm}
    \begin{tabular}{lc}
    \hline
         &  \textbf{Hyperparameter space}\\
    \hline
     \multirow{3}{*}{\textbf{NODAGS-Flow}}    &  $\log_{10}(\lambda) \in [-4, 2]$ \\
            & \# hidden units $\in \{0, 1, 2, 3\}$\\
            & $n_L \in \{5, 10, 15\}$\\
    \hline
    \textbf{LLC} & $\log_{10}(\lambda) \in [-4, 2]$ \\
    \hline
    \multirow{2}{*}{\textbf{GOLEM}} & $\log_{10}(\lambda) \in [-4, 2]$ \\
        & $\lambda_{DAG} \in [-3, 3]$\\
    \hline
    \textbf{NOTEARS} & $\log_{10}(\lambda) \in [-4, 2]$ \\
    \hline
        \textbf{DCDI} & $\log_{10}(\lambda) \in [-4, 2]$ \\
    \hline
    \end{tabular}
    \label{tab:hyper-param}
\end{table}

\subsection{Hyperparameter Tuning}
\label{sec:hyper-tune}

For all methods, an exhaustive hyperparameter search was performed using the \texttt{Ax} library \citep{bakshy2018ae} that can perform joint Bayesian and bandit optimization over the set of hyperparameters. The list of chosen hyperparameters for each model is summarized in Table \ref{tab:hyper-param}. \texttt{Ax} samples from the range of values provided for each parameter, the models are then trained using the sampled parameters and are then evaluated on the validation set. For the synthetic experiments, the training set consists of single-node interventions over all the nodes in the graph and the validation set consists of interventions over 2-3 randomly chosen nodes. For the Perturb-CITE-seq data set, we randomly split 10\% of the interventions to be a part of the validation set and the rest with the training set. We perform separate hyperparameter tuning for the synthetic and the real-world experiments and use the optimal parameter values provided by \texttt{Ax} for the respective experiments. Additionally, we fix the learning rate to $10^{-2}$ and use Adam optimizer \citep{KingmaB14} for maximizing the log-likelihood. 

\subsection{Compute Time analysis}

\begin{figure}
    \centering
    \includegraphics[width=0.8\linewidth]{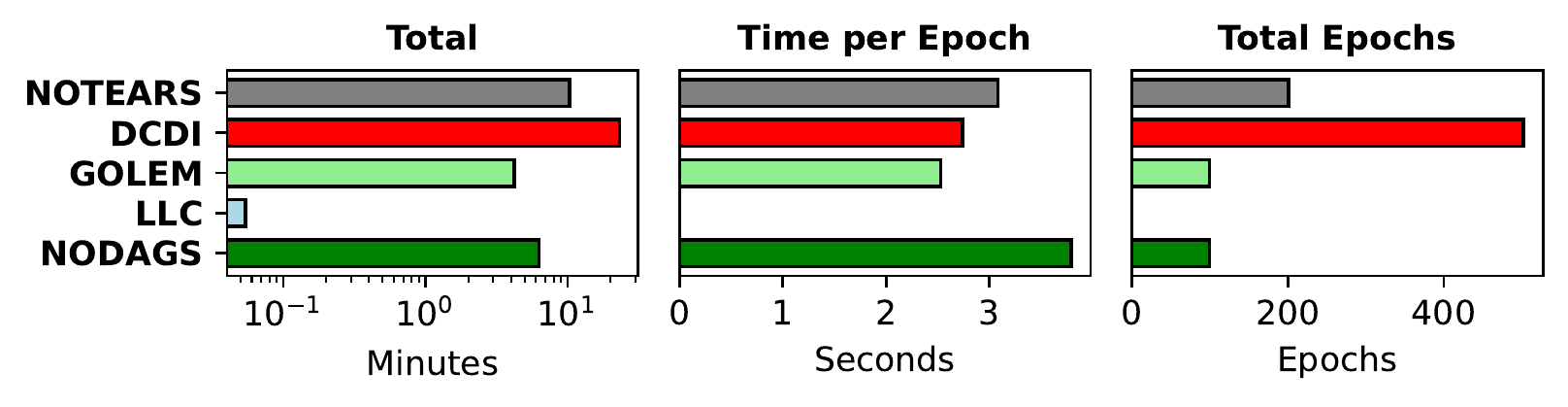}
    \caption{Comparison of the runtime of the chosen }
    \label{fig:runtime}
\end{figure}

In Figure \ref{fig:runtime}, we compare the runtime between NODAGS-Flow and the chosen baselines. It is important to note that, LLC (unlike NODAGS-Flow, DCDI, GOLEM, and NOTEARS) is not deep-learning-based method and hence doesn't require any training via stochastic gradient methods. Instead, it solves simple linear regression problems from the estimated covariance matrices for each intervention. This makes LLC considerably faster than the other algorithms as seen in Figure \ref{fig:runtime}. Additionally, due to the lack of training LLC is excluded from Time per Epoch and Total Epochs plots. Aside from LLC, we can see that the other methods are comparable in terms of total runtime and runtime per epoch. In particular, the log-det approximation necessary in every epoch of NODAGS-Flow renders the computational cost per epoch for NODAGS-Flow the highest. However, NODAGS-Flow is overall faster than the only other method relying on a nonlinear mechanism, DCDI, by a factor of more than 3.5 due to not relying on solving a constrained optimization problem via the Augmented Lagrangian Method.


\subsection{Code Statement}

We implemented our models using the PyTorch library in Python and plan to make it publicly available on GitHub upon publication. 

\section{REAL-WORLD EXPERIMENT}
\label{app:perturb-cite-seq}

The data set was downloaded from the Single Cell Portal of the Broad Institute (accession code SCP1064). We removed cells containing less than 500 expressed genes and genes that were expressed in less than 500 cells. Due to computational constraints, we chose a subset of 61 genes (Table \ref{tab:genes}) from the total set of genes in the genome, ensuring that all the chosen genes were perturbed. The three different conditions (co-culture, IFN-$\gamma$, and control) were partitioned into distinct data sets. The models were trained and evaluated on each of the three data sets. Figures \ref{fig:cluster-map-ifn} and \ref{fig:cluster-map-control} show the cluster map of the learnt adjacency matrices for the IFN-$\gamma$ and control datasets respectively. 

\begin{table}[!ht]
    \centering
    \caption{The list of chosen genes from Perturb-CITE-seq dataset \citep{frangieh2021multimodal}.}
    \vspace{0.2cm}
    \begin{tabular}{|llllllllll|}
    \hline
        ACSL3 & ACTA2 & B2M & CCND1 & CD274 & CD58 & CD59 & CDK4 & CDK6  & ~ \\
        CDKN1A & CKS1B & CST3 & CTPS1 & DNMT1 & EIF3K & EVA1A & FKBP4 & FOS  & ~ \\ 
        GSEC & GSN & HASPIN & HLA-A & HLA-B & HLA-C & HLA-E & IFNGR1 & IFNGR2  & ~ \\ 
        ILF2 & IRF3 & JAK1 & JAK2 & LAMP2 & LGALS3 & MRPL47 & MYC & P2RX4  & ~ \\ 
        PABPC1 & PAICS & PET100 & PTMA & PUF60 & RNASEH2A & RRS1 & SAT1 & SEC11C  & ~ \\ 
        SINHCAF & SMAD4 & SOX4 & SP100 & SSR2 & STAT1 & STOM & TGFB1 & TIMP2  & ~ \\ 
        TM4SF1 & TMED10 & TMEM173 & TOP1MT & TPRKB & TXNDC17 & VDAC2 & ~ &   & ~ \\ \hline
    \end{tabular}
    \label{tab:genes}
\end{table}

\subsection{NODAGS-Flow vs. LLC}
Figure \ref{fig:pert-cite-seq-llc} shows the performance comparison between NODAGS-Flow and LLC. It can be seen that NODAGS-Flow is able to outperform LLC with respect to both the evaluation metrics and across all three conditions. This shows that learning nonlinear relations does indeed provide an advantage, but we also attribute LLC's poor performance to the mismatch in LLC's treatment of the intervened nodes to its actual behavior. That is, LLC considers the intervened nodes to be independent with zero mean and unit variance, which is not the case in the data set and hence the significantly worse performance compared to the other baselines. 

\begin{figure}[h]
    \centering
    \includegraphics[width=0.8\linewidth]{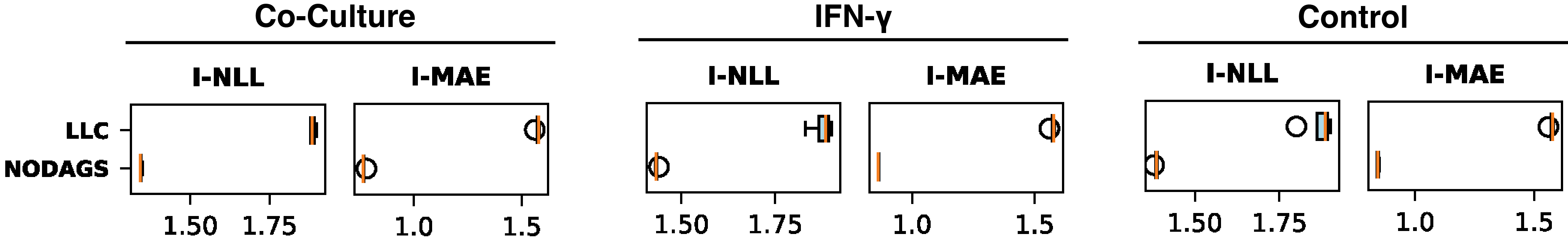}
    \caption{Performance comparison between NODAGS-Flow and LLC on Perturb-CITE-seq data set. }
    \label{fig:pert-cite-seq-llc}
\end{figure}

\begin{figure}[h]
     \centering
     \includegraphics[width=0.35\linewidth]{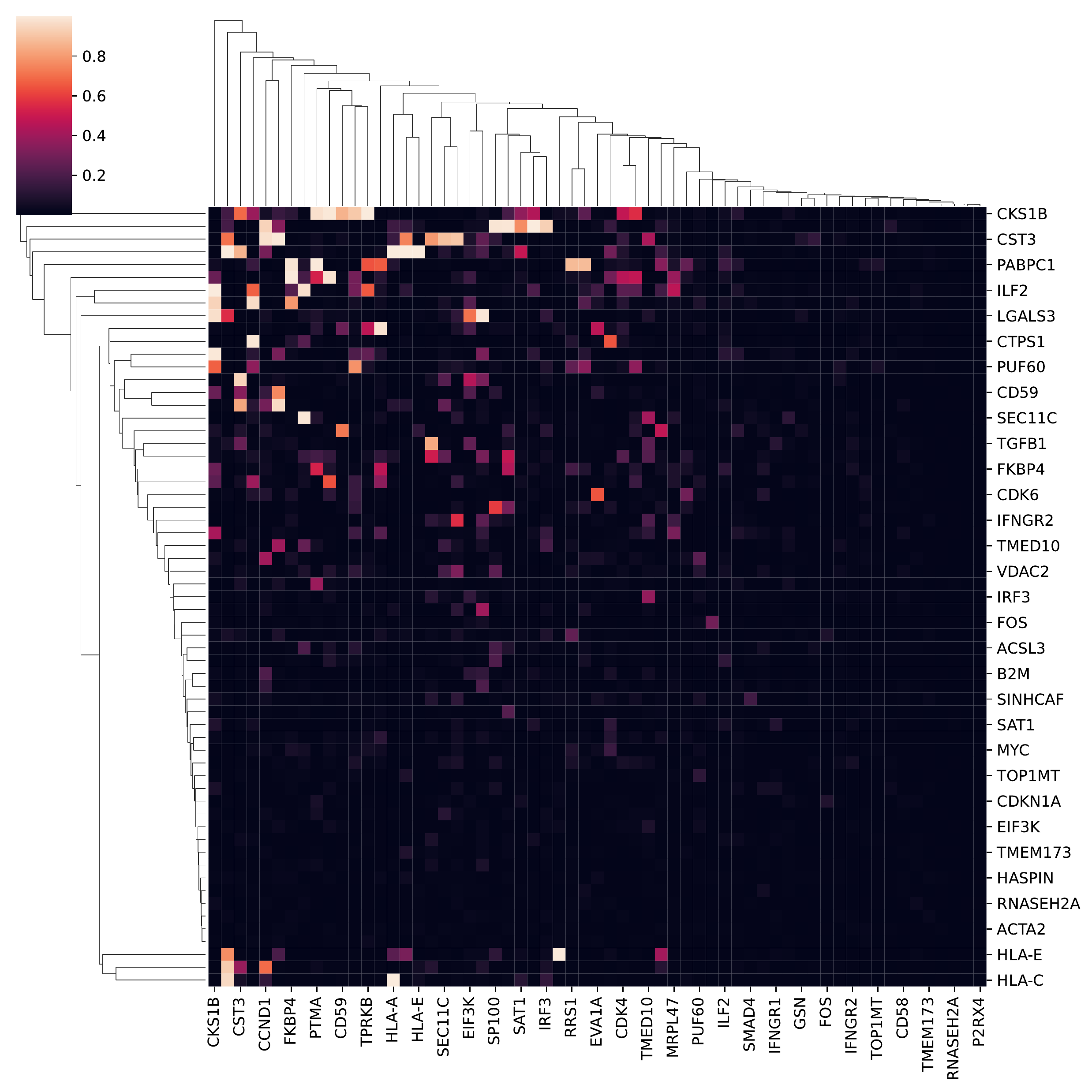}
\caption{Clustermap of the adjacency matrix learned by NODAGS-Flow on the IFN-$\gamma$ datasets. }
\label{fig:cluster-map-ifn}
\end{figure}

\begin{figure}[h]
     \centering
     \includegraphics[width=0.35\linewidth]{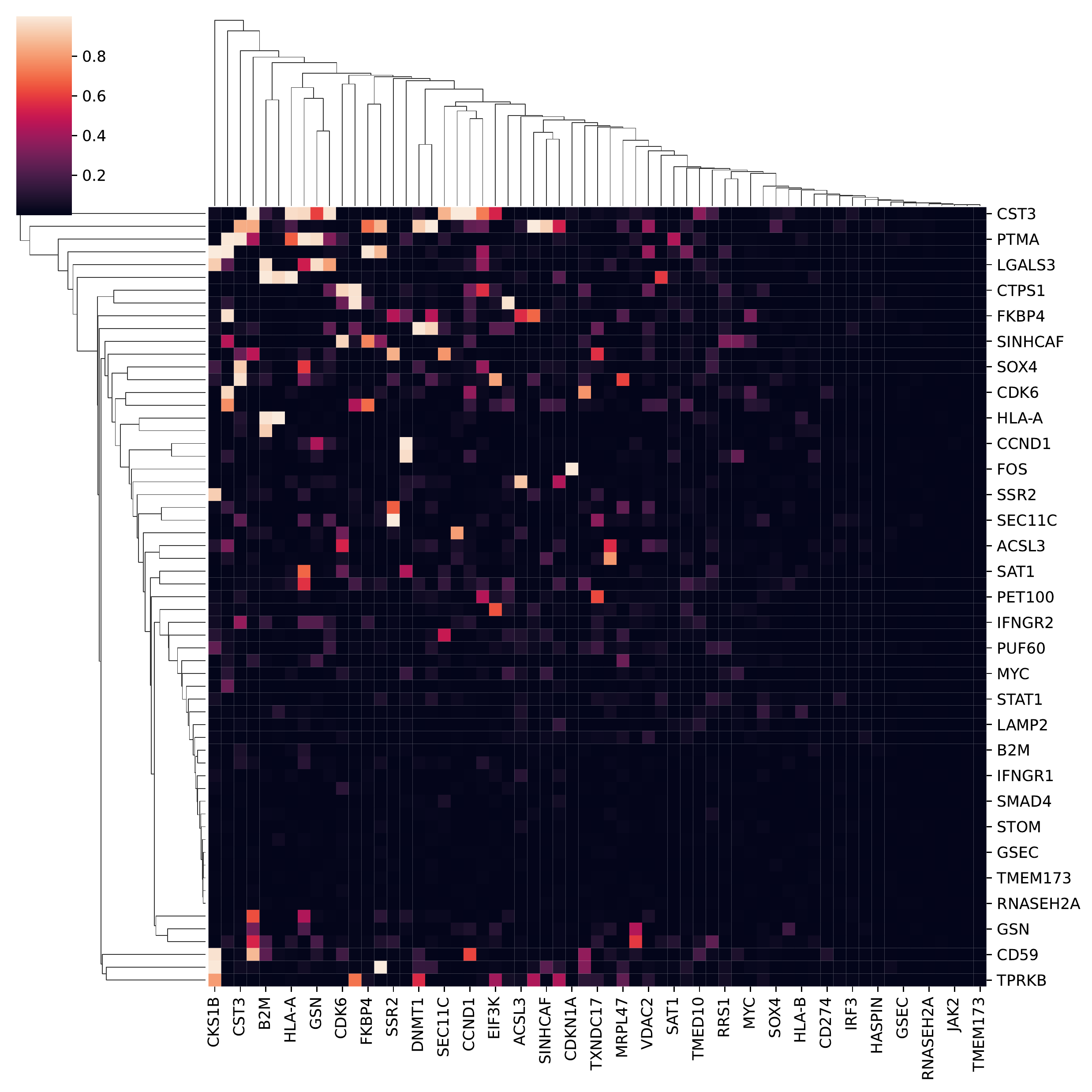}
\caption{Clustermap of the adjacency matrix learned by NODAGS-Flow on the control datasets}
\label{fig:cluster-map-control}
\end{figure}

\end{document}